\documentclass{article} 
\usepackage{nips12submit_e,times}

\usepackage{amssymb,amsfonts,amsmath,amsthm,amscd,dsfont,mathrsfs}
\usepackage{graphicx,float,psfrag,epsfig,amssymb}
\usepackage{wrapfig}
\usepackage{relsize}
\usepackage{hyperref}
\usepackage{algorithm}
\usepackage[noend]{algorithmic}
\usepackage{caption}

\newtheorem{propo}{Proposition}[section]
\newtheorem{lemma}[propo]{Lemma}
\newtheorem{definition}[propo]{Definition}
\newtheorem{coro}[propo]{Corollary}
\newtheorem{thm}[propo]{Theorem}

\newcommand{\aj}[1]{\textcolor{red}}

\def\cL{{\mathcal{L}}}
\def\hG{\widehat{G}}
\def\hH{\widehat{H}}
\def\hTh{\widehat{\Theta}}
\def\P{P}
\def\R{R}
\def\d{\rm{d}}
\def\eps{\epsilon}

\def \E{\mathbb{E}}
\def\reals{\mathbb{R}}

\def\dist{{\sf d}}
\def\alg{{\sc Algorithm }}
\def\trans{{\sf T}}

\def\tL{\widetilde{L}}

\def\normal{{\sf N}}

\def\cN{{\cal N}}
\def\tA{\widetilde{A}}
\def\tB{\widetilde{B}}
\def\tTheta{\widetilde{\Theta}}
\def\F{\mathcal{F}}
\def\P{\mathbb{P}}
\def\prob{\mathbb{P}}
\def\ind{\mathbb{I}}

\def\Cmin{C_{\rm min}}
\def\polylog{{\rm polylog}}
\def\i{^{(i)}}

\title{Efficient Reinforcement Learning for High Dimensional Linear Quadratic Systems}

\author{
Morteza Ibrahimi
\\
Stanford University\\
Stanford, CA 94305 \\
\texttt{ibrahimi@stanford.edu} \\
\And
Adel Javanmard \\
Stanford University\\
Stanford, CA 94305 \\
\texttt{adelj@stanford.edu} \\
\And
Benjamin Van Roy \\
Stanford University\\
Stanford, CA 94305 \\
\texttt{bvr@stanford.edu} 
}

\begin{document}

\maketitle

%
%
\begin{abstract}
We study the problem of adaptive control of a high dimensional linear quadratic (LQ) system. 
Previous work established the asymptotic convergence to an optimal controller for various adaptive control schemes.
More recently, for the average cost LQ problem, a regret bound of ${O}(\sqrt{T})$ was shown, apart form logarithmic factors.
However, this bound scales exponentially with $p$, the dimension of the state space.
In this work we consider the case where the matrices describing the dynamic of the LQ system are sparse and their dimensions are large.
We present an adaptive control scheme that achieves a regret bound of ${O}(p \sqrt{T})$, apart from logarithmic factors.
In particular, our algorithm has an average cost of $(1+\eps)$ times the optimum cost after $T = \polylog(p) O(1/\eps^2)$.
This is in comparison to previous work on the dense dynamics where the algorithm requires time that scales exponentially with dimension in order to achieve regret of $\eps$ times the optimal cost.

We believe that our result has prominent applications in the emerging area of computational advertising, in particular targeted online advertising and advertising in social networks.
\end{abstract}

\section{Introduction}
In this paper we address the problem of adaptive control of a high dimensional linear quadratic (LQ) system. 
Formally, the dynamics of a linear quadratic system are given by
\begin{eqnarray}\label{eq:lqr_dynamics}
x(t+1) &=& A^0 x(t) + B^0 u(t) + w(t+1), \nonumber\\
c(t) &=& x(t)^T Qx(t) + u(t)^T R u(t),
\end{eqnarray} 
where $u(t) \in \reals^r$ is the control (action) at time $t$, $x(t) \in \reals^p$ is the state at time $t$, $c(t) \in \reals$ is the cost at time $t$, and $\{w(t+1)\}_{t\ge0}$ is a sequence of random vectors in $\reals^p$ with i.i.d. standard Normal entries. 
The matrices $Q\in \reals^{p \times p}$ and $R \in \reals^{ r \times r}$ are positive semi-definite (PSD) matrices that determine the cost at each step. 
The evolution of the system is described through the matrices $A^0 \in \reals^{p \times p}$ and $B^0 \in \R^{p \times r}$.
Finally by high dimensional system we mean the case where $p, r \gg 1$.

A celebrated fundamental theorem in control theory asserts that the above LQ system can be optimally controlled by a simple linear feedback if the pair $(A^0,B^0)$ is controllable and the pair $(A^0,Q^{1/2})$ is observable. 
The optimal controller can be explicitly computed from the matrices describing the dynamics and the cost. 
Throughout this paper we assume that controllability and observability conditions hold.

When the matrix $\Theta^0 \equiv [A^0, B^0]$ is unknown, the task is that of adaptive control, where the system is to be learned and controlled at the same time.
Early works on the adaptive control of LQ systems relied on the \emph{certainty equivalence principle} \cite{bar1974dual}.
In this scheme at each time $t$ the unknown parameter $\Theta^0$ is estimated based on the observations collected so far and the optimal controller for the estimated system is applied. 
Such controllers are shown to converge to an optimal controller in the case of minimum variance cost, however, in general they may converge to a suboptimal controller \cite{guo1991aastrom}.
Subsequently, it has been shown that introducing random exploration by adding noise to the control signal, e.g., \cite{lai1982least}, solves the problem of converging to suboptimal estimates. 
%

All the aforementioned work have been concerned with the asymptotic convergence of the controller to an optimal controller. 
In order to achieve regret bounds, cost-biased parameter estimation \cite{kumar1982new, campi1997achieving, abbasiregret}, in particular the optimism in the face of uncertainty (OFU) principle \cite{lai1985asymptotically} has been shown to be effective.
In this method a \emph{confidence set} $S$ is found such that $\Theta^0 \in S$ with high probability.
The system is then controlled using the \emph{most optimistic} parameter estimates, i.e., $\hTh \in S$ with the smallest optimum cost.
The asymptotic convergence of the average cost of OFU for the LQR problem was shown in \cite{bittanti2006adaptive}.
This asymptotic result was extended in \cite{abbasiregret} by providing a bound for the cumulative regret. 
Assume $x(0) = 0$ and for a control policy $\pi$ define the average cost
\vspace{-5pt}
\begin{eqnarray}
J_{\pi} = \underset{T \to \infty}{\rm limsup}\, \frac{1}{T}\sum_{t=0}^T \E[c_t]\,.
\end{eqnarray}
Further, define the cumulative regret as 
\vspace{-5pt}
\begin{eqnarray}
R(T) = \sum_{t=0}^{T} (c_{\pi}(t) - J_*)\,, 
\end{eqnarray}
where $c_{\pi}(t)$ is the cost of control policy $\pi$ at time $t$ and $J_* = J(\Theta_0)$ is the optimal  average cost.
The algorithm proposed in \cite{abbasiregret} is shown to have cumulative regret $\tilde{O}(\sqrt{T})$ where $\tilde{O}$ is hiding the logarithmic factors.
While no lower bound was provided for the regret, comparison with the multi-armed bandit problem where a lower bound of $O(\sqrt{T})$ was shown for the general case \cite{dani2008stochastic}, suggests that this scaling with time for the cumulative regret is optimal.


The focus of~\cite{abbasiregret} was on scaling of the regret with time horizon $T$. However, the regret of the proposed algorithm scales poorly with dimension. More specifically, the analysis in~\cite{abbasiregret} proves a regret bound of $R(T) < Cp^{p+r+2} \sqrt{T}$. The current paper focuses on (many) applications where the state and control dimensions are much larger than the time horizon of interest.
A powerful reinforcement learning algorithm for these applications should have regret which depends gracefully on dimension. 
In general, there is little to be achieved when $T < p$ as the number of \emph{degrees of freedom} ($pr+p^2$) is larger than the number of observations ($Tp$) and any estimator can be arbitrary inaccurate.
However, when there is prior knowledge about the unknown parameters $A^0, B^0$, \emph{e.g.}, when $A^0, B^0$ are \emph{sparse}, accurate estimation can be feasible.
In particular, \cite{bento2010learning} proved that under suitable conditions the unknown parameters of a noise driven system (i.e., no control) whose dynamics are modeled by linear stochastic differential equations can be estimated accurately with as few as $O(\log(p))$ samples. 
However, the result of \cite{bento2010learning} is not directly applicable here since for a general feedback gain $L$ even if $A^0$ and $B^0$ are sparse, the closed loop gain $A^0-B^0 L$ need not be sparse. 
Furthermore, system dynamics would be correlated with past observations through the estimated gain matrix $L$. Finally, there is no notion of cost in~\cite{bento2010learning} while here we have to obtain bounds on cost and its scaling with p.
In this work we extend the result of \cite{bento2010learning} by showing that under suitable conditions, unknown parameters of sparse high dimensional LQ systems can be accurately estimated with as few as $O(\log(p+r))$ observations.
Equipped with this efficient learning method, we show that sparse high dimensional LQ systems can be adaptively controlled with regret $\tilde{O}(p \sqrt{T})$.

To put this result in perspective note that even when $x(t) = 0$, the expected cost at time $t+1$ is $\Omega(p)$ due to the noise. 
Therefore, the cumulative cost at time $T$ is bounded as $\Omega(pT)$.
Comparing this to our regret bound, we see that for $T = \polylog(p) {O}(\frac{1}{\eps^2})$, the cumulative cost of our algorithm is bounded by $(1+\eps)$ times the optimum cumulative cost.
In other words, our algorithm performs close to optimal after $\polylog(p)$ steps.
This is in contrast with the result of \cite{abbasiregret} where the algorithm needs $\Omega(p^{2p})$ steps in order to achieve regret of $\eps$ times the optimal cost.

Sparse high dimensional LQ systems appear in many engineering applications. 
Here we are particularly motivated by an emerging field of applications in marketing and advertising. 
The use of dynamical optimal control models in advertising has a history of at least four decades, cf. \cite{sethi1977dynamic, feichtinger1994dynamic} for a survey.
In these models, often a partial differential equation is used to describe how advertising expenditure over time translates into sales.
The basic problem is to find the advertising expenditure that maximizes the net profit.
%
The focus of these works is to model the temporal dynamics of the advertising expenditure (the control variable) and the variables of interest (sales, goodwill level, etc.).
There also exists a rich literature studying the \emph{spatial} interdependence of consumers' and firms' behavior to devise marketing schemes \cite{bradlow2005spatial}.
In these models space can be generalized beyond geographies to include notions like demographies and psychometry.

Combination of spatial interdependence and temporal dynamics models for optimal advertising was also considered \cite{seidman1987dynamics, marinelli2008optimal}.
A simple temporal dynamics model is extended in \cite{marinelli2008optimal} by allowing state and control variables to have spatial dependence and introducing a diffusive component in the controlled PDE which describes the spatial dynamics.
The controlled PDE is then showed to be equivalent to an abstract linear control system of the form
\begin{equation}
  \frac{\d x(t)}{\d t} =  A x(t) + B u(t).
  \label{eq:noiseless_LQR}
\end{equation}
%

Both \cite{marinelli2008optimal} and \cite{bradlow2005spatial} are concerned with the optimal control and the interactions are either dictated by the model or assumed known.
Our work deals with a discrete and noisy version of \eqref{eq:noiseless_LQR} where the dynamics is to be estimated but is known to be sparse.
In the model considered in \cite{marinelli2008optimal} the state variable $x$ lives in an infinite dimensional space. 
Spatial models in marketing \cite{bradlow2005spatial} usually consider state variables which have a large number of dimensions, e.g., number of zip codes in the US ($\sim 50$K).
High dimensional state space and control is a recurring theme in these applications.

In particular, with the modern social networks customers are classified in a highly granular way, potentially with each customer representing his own class.
With the number of classes and complexity of their interactions, its unlikely that we could formulate an effective model a priori for how classes interact.  Further, the nature of these interactions change over time with the changing landscape of Internet services and information available to customers.
This makes it important to efficiently learn from real-time data about the nature of these interactions.


%
%

\noindent {\bf Notation:}
We bundle the unknown parameters into one variable $\Theta^0 = [A^0, B^0] \in \reals^{p\times q}$ where $q = p+r$ and call it the interaction matrix.
For $v \in \reals^n$, $M \in \reals^{m \times n}$ and $p \ge 1$, we denote by $\|v\|_p$ the standard p-norm and by $\|M\|_p$ the corresponding operator norm. For $1\le i \le m$,
$M_i$ represents the $i^{th}$ row of matrix $M$.
For $S \subseteq [m], J \subseteq [n]$, $M_{SJ}$ is the submatrix of $M$ formed by the rows in $S$ and columns in $J$. 
For a set $S$ denote by $|S|$ its cardinality. For an integer $n$ denote by $[n]$ the set $\{1, \dotsc, n\}$.
\section{Algorithm}
Our algorithm employs the \emph{Optimism in the Face of Uncertainty} (OFU) principle in an episodic fashion. 
At the beginning of episode $i$ the algorithm constructs a \emph{confidence set} $\Omega\i$ which is guaranteed to include the unknown parameter $\Theta^0$ with high probability.
The algorithm then chooses $\tTheta\i \in \Omega\i$ that has the smallest expected cost as the estimated parameter for episode $i$ and applies the optimal control for the estimated parameter during episode $i$.

The confidence set is constructed using observations from the last episode only but the length of episodes are chosen to increase geometrically allowing for more accurate estimates and shrinkage of the confidence set by a constant factor at each episode.
The details of each step and the pseudo code for the algorithm follows.


{\bf Constructing confidence set:}
Define $\tau_i$ to be the start of episode $i$ with $\tau_0 = 0$.
Let $L\i$ be the controller that has been chosen for episode $i$.
For $ t \in [\tau_i , \tau_{i+1})$ the system is controlled by $u(t) =  - L\i x(t)$ and the system dynamics can be written as
%
$x(t+1) = (A^0 - B^0 L\i) x(t) + w(t+1). $
%
At the beginning of episode $i+1$, first an initial estimate $\hTh$ is obtained by solving the following convex optimization problem for each row $\Theta_u \in \reals^q$ separately:
\begin{eqnarray}\label{eq:theta_hat_def}
  \hTh_u^{(i+1)} \in \text{argmin} \,\, \mathcal{L}(\Theta_u) + \lambda \|\Theta_u\|_1, 
\end{eqnarray}
where 
\begin{eqnarray}\label{eq:likelihood_function}
\mathcal{L}(\Theta_u) = \frac{1}{2 \Delta \tau_{i+1}} \sum_{t=\tau_i}^{\tau_{i+1}-1} \{x_u(t+1) - \Theta_{u} \tL^{(i)} x(t)\}^2, \quad \Delta\tau_{i+1} = \tau_{i+1} - \tau_i,
\end{eqnarray}
and $\tL\i = [I, - L^{(i)\trans}]^\trans$. 
The estimator $\hTh_u$ is known as the LASSO estimator. 
The first term in the cost function is the normalized negative log likelihood which measures the fidelity to the observations while the second term imposes the sparsity constraint on $\Theta_u$. 
$\lambda$ is the regularization parameter.

For $\Theta^{(1)}, \Theta^{(2)} \in \reals^{p\times q}$ define the distance $\dist(\Theta^{(1)}, \Theta^{(2)})$ as
\begin{equation}
\dist(\Theta^{(1)}, \Theta^{(2)}) = \underset{u\in[p]}{\max}\, \|\Theta^{(1)}_u - \Theta^{(2)}_u\|_2,
\end{equation}
where $\Theta_u$ is the $u^{th}$ row of the matrix $\Theta$.
It is worth noting that for $k$-sparse matrices with $k$ constant, this distance does not scale with $p$ or $q$.
In particular, if the absolute value of the elements of $\Theta^{(1)}$ and $\Theta^{(2)}$ are bounded by $\Theta_{\max}$ then $\dist(\Theta^{(1)}, \Theta^{(2)}) \le 2\sqrt{k} \Theta_{\max}$.

Having the estimator $\hTh\i$ the algorithm constructs the confidence set for episode $i$ as 
\begin{equation}
  \Omega\i = \{\Theta \in \reals^{p\times q} \,|\; \dist(\Theta, \hTh\i) \le 2^{-i} \eps \},
\end{equation}
where $\eps >0$ is an input parameter to the algorithm.
For any fixed $\delta > 0 $, by choosing $\tau_i$ judiciously we ensure that with probability at least $1-\delta$, $\Theta^0 \in \Omega\i$, for all $i \ge 1$. (see Theorem~\ref{thm:learning}). 


\begin{algorithm}[t]
\caption*{{\sc Algorithm}:  Reinforcement learning algorithm for LQ systems.}
\begin{algorithmic}[1]

\REQUIRE Precision $\eps$, failure probability $4\delta$, initial $(\rho, C_{\min},\alpha)$ identifiable controller $L^{(0)}$, $\ell(\Theta^0,\eps)$

\ENSURE Series of estimates $\tTheta\i$, confidence sets $\Omega\i$ and controllers $L\i$

\STATE Let $\ell_0 = \max(1,\max_{j\in [r]} \|L_j^{(0)}\|_2)$, and
\begin{align*}
n_0 &= \frac{4\cdot 10^3\, k^2\ell_0^2}{\alpha(1-\rho) C_{\min}^2} \left( \frac{1}{\eps^2} + \frac{k}{(1-\rho)^2}\right) \log(\frac{4kq}{\delta}),\\
n_1 &= \frac{4\cdot 10^3\, k^2\ell(\Theta^0,\eps)^2}{(1-\rho) C_{\min}^2} \left( \frac{1}{\eps^2} + \frac{k}{(1-\rho)^2}\right) \log(\frac{4kq}{\delta}).
\end{align*}

Let $\Delta \tau_0 = n_0$, $\Delta \tau_i  = 4^i(1+i/\log(q/\delta)) n_1$ for $i\ge 1$, and $\tau_i = \sum_{j=0}^{i} \Delta \tau_j$.

\FOR{$i = 0,1,2, \dots$}

\STATE Apply the control $u(t) = -L\i x(t)$ until $\tau_{i+1}-1$ and observe the trace $\{x(t)\}_{\tau_i\le t < \tau_{i+1}}$.

\STATE Calculate the estimate $\hTh^{(i+1)}$ from \eqref{eq:theta_hat_def} and construct the confidence set $\Omega^{(i+1)}$.

\STATE Calculate $\tTheta^{(i+1)}$ from \eqref{eqn:choiceoftheta} and set $L^{(i+1)} \leftarrow L(\tTheta^{(i+1)})$.

\ENDFOR

\end{algorithmic}
\end{algorithm}

{ \bf Design of the controller:}
Let $J(\Theta)$ be the minimum expected cost if the interaction matrix is $\Theta = [A, B]$ and denote by $L(\Theta)$ the optimal controller that achieves the expected cost $J(\Theta)$.
The algorithm implements OFU principle by choosing, at the beginning of episode $i$, an estimate $\tTheta\i \in \Omega\i$ such that
\begin{eqnarray}\label{eqn:choiceoftheta}
  \tTheta\i \in \underset{\Theta \in \Omega\i}{\text{argmin}}\, J(\Theta).
\end{eqnarray}

The optimal control corresponding to $\tTheta\i$ is then applied during episode $i$, i.e., $u(t) = -L(\tTheta\i) x(t)$ for $ t \in [\tau_i, \tau_{i+1})$. 
Recall that for $\Theta = [A, B]$, the optimal controller is given through the following relations
\begin{align*}
K(\Theta) &= Q + A^\trans K(\Theta) A - A^\trans K(\Theta)B( B^\trans K(\Theta)B + R)^{-1} B^\trans K(\Theta)A\,,\;\; \text{(Riccati equation)}\\
L(\Theta)  &=(B^\trans K(\Theta)B+R)^{-1} B^\trans K(\Theta)A\,.
\end{align*}
%
The pseudo code for the algorithm is summarized in the table.

\section{Main Results}

In this section we present performance guarantees in terms of cumulative regret and learning accuracy for the presented algorithm. In order to state the theorems, we first need to present some assumptions on the system.
 
Given $\Theta \in \reals^{p \times q}$ and $L \in \reals^{r\times p}$, define $\tL = [I,-L^\trans]^\trans\in \reals^{q\times p}$ and let $\Lambda \in \reals^{p\times p}$ be a solution to the following Lyapunov equation
\begin{equation}\label{eq:lyapunov_equation}
\Lambda - \Theta \tL \Lambda \tL^\trans \Theta^\trans  = I. 
\end{equation}
If the closed loop system $(A^0-B^0 L)$ is stable then the solution to the above equation exists and the state vector $x(t)$ has a Normal stationary distribution with covariance $\Lambda$.

We proceed by introducing an \textit{identifiable regulator}. 
\begin{definition}\label{def:identifiable-reg}
For a $k$-sparse matrix $\Theta^0 = [A^0, B^0] \in \reals^{p\times q}$  and $L \in \reals^{r\times p}$, define $\tL = [ I, -L^\trans]^\trans \in \reals ^{q\times p}$ and let $H =\tL \Lambda \tL^\trans $ where $\Lambda$ is the solution of Eq. \eqref{eq:lyapunov_equation} with $\Theta = \Theta^0$. Define $L$ to be $(\rho, C_{\rm min}, \alpha)$ identifiable (with respect to $\Theta^0$) if it satisfies the following conditions for all $S \subseteq [q], \,|S|\le k$.
\begin{align*}
(1)\; \| A^0 - B^0 L\|_2 \le \rho < 1,  \quad
(2)\; \lambda_{\min}(H_{SS}) \ge C_{\min} , \quad
(3)\; \|H_{S^c S} H_{SS}^{-1} \|_{\infty} \le 1-\alpha .
\end{align*}
\end{definition}
The first condition simply states that if the system is controlled using the regulator $L$ then the closed loop autonomous system is asymptotically stable. 
The second and third conditions are similar to what is referred to in the sparse signal recovery literature as the \textit{mutual incoherence} or \emph{irreprepresentable} conditions. 
Various examples and results exist for the matrix families that satisfy these conditions~\cite{tropp2006just}.
Let $S$ be the set of indices of the nonzero entries in a specific row of $\Theta^0$. 
The second condition states that the corresponding entries in the extended state variable $ y = [x^\trans, u^\trans]$ are sufficiently distinguishable from each other.
In other words, if the trajectories corresponding to this group of state variables are observed, non of them can be \emph{well approximated} as a linear combination of the others.
The third condition can be thought of as a quantification of the first vs. higher order dependencies.
Consider entry $j$ in the extended state variable. 
Then, the dynamic of $y_j$ is directly influenced by entries $y_S$.
However they are also influenced indirectly by other entries of $y$. 
The third condition roughly states that the indirect influences are sufficiently weaker than the direct influences.
There exists a vast literature on the applicability of these conditions and scenarios in which they are known to hold.
These conditions are \emph{almost} necessary for the successful recovery by $\ell_1$ relaxation.
For a discussion on these and other similar conditions imposed for sparse signal recovery we refer the reader to \cite{wainwright2009sharp} and \cite{zhao} and the references therein.

Define $\Theta_{\min} = \min_{i \in [p], j\in [q], \Theta^0_{ij} \neq 0} |\Theta^0_{ij}|$. Our first result states that the system can be learned efficiently from its trajectory observations when it is controlled by an identifiable regulator. 
\begin{thm}\label{thm:learning}
  Consider the LQ system of Eq. \eqref{eq:lqr_dynamics} and assume $\Theta^0 = [A^0, B^0]$ is $k$-sparse.
  Let $u(t) = -L x(t)$ where $L$ is a $(\rho, C_{\rm min}, \alpha)$ identifiable regulator with respect to $\Theta^0$ and define $\ell = \max(1, \max_{j \in [r]}\|L_j\|_2)$. 
Let $n$ denote the number of samples of the trajectory that is observed. 
  For any $ 0 < \eps < \min(\Theta_{\min}, \frac{\ell}{2}, \frac{3}{1-\rho})$, there exists $\lambda$ such that, if 
\begin{equation}
  n \ge \frac{4\cdot 10^3\, k^2 \ell^2}{\alpha^2 (1-\rho) \Cmin^2} \left(\frac{1}{\eps^2}+\frac{k}{(1-\rho)^2} \right) \log(\frac{4kq}{\delta})\,,
  \label{eq:bdd_n_all}
\end{equation}
then the $\ell_1$-regularized least squares solution $\hTh$ of Eq. \eqref{eq:theta_hat_def} satisfies $\dist(\hTh, \Theta^0) \le \epsilon$ with probability larger than $1-\delta$. In particular, this is achieved by taking $\lambda = 6 \ell \sqrt{\log(4q/\delta)/(n\alpha^2 (1-\rho))}$ .
\end{thm}

Our second result states that equipped with an efficient learning algorithm, the LQ system of Eq. \eqref{eq:lqr_dynamics} can be controlled with regret $\tilde{O}(p\sqrt{T} \log^{\frac{3}{2}}({1}/{\delta}))$ under suitable assumptions. 

Define an $\epsilon$-neighborhood of $\Theta^0$ as $\cN_{\epsilon}(\Theta^0) = \{\Theta \in \reals^{p\times q} \,|\, \dist(\Theta^0, \Theta) \le \epsilon\}$. 
Our assumption asserts the identifiably of $L(\Theta)$ for $\Theta$ close to $\Theta^0$.

\noindent \textbf{Assumption:} There exist $\epsilon, C > 0$ such that for all $\Theta \in \cN_{\epsilon}(\Theta^0)$, $L(\Theta)$ is identifiable w.r.t. $\Theta^0$
and
\begin{align*}
\sigma_L(\Theta^0, \epsilon) = &\sup_{\Theta \in \cN_\epsilon(\Theta^0)} \|L(\Theta)\|_2 \le C, \quad
\sigma_K(\Theta^0, \epsilon) = \sup_{\Theta \in \cN_\epsilon(\Theta^0)} \|K(\Theta)\|_2 \le C. 
\end{align*}
Also define 
\begin{eqnarray*} 
\ell(\Theta^0, \epsilon) = \sup_{\Theta \in \cN_\epsilon(\Theta^0)} \max(1, \max_{j \in [r]}\|L_j(\Theta)\|_2)\,.
\end{eqnarray*}
Note that $\ell(\Theta^0, \epsilon) \le\max(C, 1)$, since $\max_{j \in [r]}\|L_j(\Theta)\|_2 \le \|L(\Theta)\|_2$. 

\begin{thm}\label{thm:exploitation}
Consider the LQ system of Eq. \eqref{eq:lqr_dynamics}. For some constants $\eps, C_{\rm min}$ and $0 < \alpha, \rho <1$, assume that an initial $(\rho, C_{\rm min}, \alpha)$ identifiable regulator $L^{(0)}$ is given. Further, assume that for any $\Theta \in \cN_{\epsilon}(\Theta^0)$, $L(\Theta)$ is $(\rho, C_{\rm min}, \alpha)$ identifiable. Then, with probability at least $1-\delta$ the cumulative regret of \alg (cf. the table) is bounded as
\begin{equation} 
  R(T) \le  \tilde{O}(p\sqrt{T} \log^{\frac{3}{2}}({1}/{\delta}))\,,
\end{equation}
where $\tilde{O}$ is hiding the logarithmic factors.
\end{thm}

\section{Analysis}
\subsection{Proof of Theorem~\ref{thm:learning}}
To prove theorem \ref{thm:learning} we first state a set of sufficient conditions for the solution of the $\ell_1$-regularized least squares to be within some distance, as defined by $\dist(\cdot,\cdot)$, of the true parameter.
Subsequently, we prove that these conditions hold with high probability.

Define $X= [x(0), x(1), \dots, x(n-1)] \in\reals^{p\times n}$ and let $W = [w(1),\dots,w(n)] \in \reals^{p \times n}$ be the matrix containing the Gaussian noise realization. 
Further let the $W_u$ denote the $u^{\rm th}$ row of $W$.

Define the normalized gradient and Hessian of the likelihood function \eqref{eq:likelihood_function} as
\begin{eqnarray}
\hG = - 
\nabla \cL (\Theta^0_u) =  \frac{1}{n} \tL X W_u^\trans\, ,\;\;\;\;\;\;\;
\hH =  \nabla^2 \cL (\Theta^0_u) =  \frac{1}{n} \tL X X^\trans \tL^\trans\, .
\end{eqnarray}

The following proposition, a proof of which can be found in \cite{zhao}, provides a set of sufficient conditions for the \textit{accuracy} of the $\ell_1$-regularized least squares solution.

\begin{propo}\label{th:cond_to_hold}
  Let $S$ be the support of $\Theta^0_u$ with $|S| < k$, and $H$ be defined per Definition~\ref{def:identifiable-reg}. Assume there exist $0<\alpha<1$ and $C_{\min}>0$ such that 
\begin{eqnarray}
\lambda_{\min}(H_{{S},{S}}) \ge C_{\min} \,,
\qquad
  \| H_{S^c, {S}} H_{{S},{S}}^{-1} \|_\infty \le
 1 - \alpha\, . \label{eq:cond_irrep}
\end{eqnarray}
For any $ 0 < \eps < \Theta_{\rm min}$ if the following conditions hold
\begin{align}
\|\hG\|_\infty &\leq \frac{\lambda \alpha}{3}\, ,\;\;\;\;\;\;\;\;\;\;\;\;\;
\|\hG_{S}\|_\infty \leq \frac{\epsilon C_{\min}}{4 k} - \lambda,\label{eq:Conditions1}\\
 \| \hH_{S^C {S}} - H_{S^C {S}} \|_\infty &\leq
\frac{\alpha}{12} \frac{C_{\rm min}}{\sqrt{k}}\, 
,\;\;\;\;\;\;\;
 \| \hH_{{S}{S}} - H_{{S}{S}} \|_\infty \leq \frac{\alpha}{12} \frac{C_{\min}}{\sqrt{k}}\, ,\label{eq:matrix_norm_cond_12}
\end{align}
the $\ell_1$-regularized least square solution \eqref{eq:theta_hat_def} satisfies $\dist(\hTh_u, \Theta^0_u) \le \eps$.
\end{propo}

In the sequel, we prove that the conditions in Proposition \ref{th:cond_to_hold} hold with high probability given that the assumptions of Theorem \ref{thm:learning} are satisfied. 
A few lemmas are in order proofs of which are deferred to the Appendix.

The first lemma states that $\hG$ concentrates in infinity norm around its mean of zero.
\begin{lemma} \label{th:gradie_prob_bound}
Assume $\rho = \|A^0-B^0L\|_2 < 1$ and let $\ell = \max(1, \max_{i \in [r]}\|L_i\|_2)$. 
Then, for any $S\subseteq [q]$ and $0 < \epsilon < \frac{\ell}{2}$
\begin{eqnarray}
\prob\big\{\|\hG_S\|_\infty > \epsilon\big\} \leq 2|S| \, \exp \left(- \frac{n(1-\rho) \epsilon^2 }{4 \ell^2}\right) \, .
\end{eqnarray}
\end{lemma}
To prove the conditions in Eq. \eqref{eq:matrix_norm_cond_12} we first bound in the following lemma the absolute deviations of the elements of $\hH$ from their mean $H$, i.e.,  $|\hH_{ij} - H_{ij}|$.
\begin{lemma}\label{en:bddQij}
 Let $i,j\in [q]$,  $ \rho = \|A^0-B^0L\|_2 < 1$, and $0 < \epsilon < \frac{3}{1-\rho} < n$\,. Then, 
\begin{eqnarray}
\prob(|\hH_{ij} - H_{ij}|  > \epsilon) \leq   2 \exp\left(-\frac{n (1-\rho)^3 \eps^2}{24 \ell^2} \right).
\end{eqnarray} 
\end{lemma}
The following corollary of Lemma \ref{en:bddQij} bounds $ \|\hH_{JS} - {H}_{JS}  \|_\infty$  for $J, S \subseteq [q]$. 
\begin{coro}\label{en:bddQJS}
  Let $J,S \subseteq [q]$, $ \rho = \|A^0-B^0L\|_2 <1$, $\epsilon< \frac{3|S|}{1-\rho}$, and $n > \frac{3}{1-\rho}$. Then, 
\begin{eqnarray}
\prob( \| \hH_{JS} - H_{JS} \| _{\infty}  > \epsilon) \leq   2 |J| |S| \exp\left(-\frac{n (1-\rho)^3 \eps^2}{24 |S|^2 \ell^2} \right).
\end{eqnarray} \label{th:prob_bound_inft_matrix_hess}
\end{coro}
The proof of Corollary~\ref{en:bddQJS} is by applying union bound as 
\begin{equation}\label{eq:UnionBound}
\prob ( \|  \hH_{JS} - H_{JS}  \| _{\infty} > \epsilon )
\leq |J| |S| \max_{i \in J, j \in S} \prob ( | \hH_{ij} - H_{ij} | > {\epsilon}/{ |S|}).
\end{equation}

\begin{proof}[Proof of Theorem \ref{thm:learning}]
We show that the conditions given by Proposition \ref{th:cond_to_hold} hold.
The conditions in Eq. \eqref{eq:cond_irrep} are true by the assumption of identifiability of $L$ with respect to $\Theta^0$. 
In order to make the first constraint on $\hG$ imply the second constraint on $\hG$, we assume that $\lambda \alpha /3 \le \eps C_{\min} /(4k) - \lambda$, which is ensured to hold if $\lambda \le \eps C_{\min} / (6k)$. 
By Lemma \ref{th:gradie_prob_bound}, $\prob(\|\hG\|_{\infty} > {\lambda \alpha/3}) \le \delta /2$ if 
\begin{equation}
  \lambda^2 =  \frac{36 \ell^2}{n (1-\rho)\alpha^2} \log(\frac{4q}{\delta})\,.
  \label{eq:bdd_n_G}
\end{equation}
\vspace{-1mm}
Requiring $\lambda \le \eps C_{\min} / (6k)$, we obtain 
\begin{equation}
n \ge \frac{36^2\, k^2 \ell^2}{\eps^2 \alpha^2 C_{\min}^2 (1-\rho)} \log(\frac{4q}{\delta})\,.
  \label{eq:bdd_n_G}
\end{equation}
The conditions on $\hH$ can also be aggregated as 
%
 $\|\hH_{[q], {S}} - H_{[q], {S}} \|_\infty \leq \alpha C_{\min} / (12\sqrt{k})\,$.
%
By Corollary \ref{en:bddQJS}, $\prob(\|\hH_{[q]S} - H_{[q]S}\|_{\infty} > \alpha C_{\min} / (12\sqrt{k})) \le {\delta}/{2}$ if 
\begin{equation}
n \ge \frac{3456\, k^3 \ell^2}{\alpha^2(1-\rho)^3 \Cmin^2} \log(\frac{4kq}{\delta}).
\label{eq:bdd_n_H}
\end{equation}
Merging the conditions in Eq. \eqref{eq:bdd_n_G} and \eqref{eq:bdd_n_H} we conclude that the conditions in Proposition \ref{th:cond_to_hold} hold with probability at least $1-\delta$ if 
\begin{equation}
  n \ge \frac{4 \cdot 10^3\, k^2 \ell^2}{\alpha^2 (1-\rho) \Cmin^2} \left(\frac{1}{\eps^2}+\frac{k}{(1-\rho)^2} \right) \log(\frac{4kq}{\delta}).
  \label{eq:bdd_n_all}
\end{equation}
Which finishes the proof of Theorem \ref{thm:learning}.
\end{proof}

\subsection{Proof of Theorem~\ref{thm:exploitation}}
The high-level idea of the proof is similar to the proof of main Theorem in~\cite{abbasiregret}. First, we give a decomposition for the gap between the cost obtained by the algorithm and the optimal cost. We then upper bound each term of the decomposition separately. 
 
\subsubsection{Cost Decomposition}
 Writing the Bellman optimality equations~\cite{bertsekas07,bertsekas1987dynamic} for average cost dynamic programming, we get
\begin{align*}
J(\tTheta_t) + x(t)^{\trans} K(\tTheta_t) x(t) 
 = \min_{u} \bigg\{ x(t)^{\trans} Q x(t) + u^{\trans} R u +  \E \big[ z(t+1)^{\trans} K(\tTheta_t) z(t+1) \arrowvert \F_t \big] \bigg\},
\end{align*}
where $\tTheta_t = [\tA, \tB]$ is the estimate used at time $t$, $z(t+1) = \tA_t x(t) + \tB_t u + w(t+1)$, and $\F_t$ is the $\sigma$-field generated by the variables $\{(z_\tau,x_\tau)\}_{\tau=0}^t$. Notice that the left-hand side is the average cost occurred with initial state $x(t)$~\cite{bertsekas07,bertsekas1987dynamic}. Therefore,
\begin{align*}
 J(\tTheta_t) +  x(t)^{\trans} &K(\tTheta_t) x(t) = x(t)^{\trans} Q x(t) + u(t)^{\trans} R u(t) \\
&+ \E \big[ (\tA_t x(t)+ \tB_t u(t) + w(t+1))^{\trans} K(\tTheta_t) (\tA_t x(t)+ \tB_t u(t) + w(t+1)) \arrowvert \F_t \big]\\
&=  x(t)^{\trans} Q x(t) + u(t)^{\trans} R u(t)  + \E \big[ (\tA_t x(t)+ \tB_t u(t) )^{\trans} K(\tTheta_t) (\tA_t x(t)+ \tB_t u(t) ) \arrowvert \F_t \big]\\
& \quad +  \E \big[ w(t+1)^{\trans} K(\tTheta_t) w(t+1)\arrowvert \F_t]\\
& = x(t)^{\trans} Q x(t) + u(t)^{\trans} R u(t) + \E \big[ x(t+1)^{\trans} K(\tTheta_t) x(t+1) \arrowvert \F_t \big]\\
&\quad +  \Big( (\tA_t x(t)+ \tB_t u(t) )^{\trans} K(\tTheta_t) (\tA_t x(t)+ \tB_t u(t) ) \\
& \quad -  (A^0 x(t)+ B^0 u(t) )^{\trans} K(\tTheta_t) (A^0 x(t)+ B^0 u(t) ) \Big).
\end{align*}
Consequently
\begin{align}\label{eqn:cost_decomposition}
\sum_{t=0}^{T} \big(x(t)^{\trans} Qx(t) + u(t)^{\trans} R u(t) \big) &=
\sum_{t=0}^{T} J(\tTheta_t) + C_1 + C_2 + C_3,
\end{align}
where
\begin{align}
C_1 &= \sum_{t=0}^{T} \bigg(x(t)^{\trans} K(\tTheta_t)x(t) - \E \big[x(t+1)^{\trans} K(\tTheta_{t+1}) x(t+1) \big|\F_t \big] \bigg), \\
C_2 &= - \sum_{t=0}^{T} \E \big[x(t+1)^{\trans} (K(\tTheta_t) - K(\tTheta_{t+1})) x(t+1) \big|\F_t \big], \\
C_3 &= - \sum_{t=0}^{T} \Big( (\tA_t x(t)+ \tB_t u(t) )^{\trans} K(\tTheta_t) (\tA_t x(t)+ \tB_t u(t) )\nonumber \\
& \hspace{50pt} - (A^0 x(t)+ B^0 u(t) )^{\trans} K(\tTheta_t) (A^0 x(t)+ B^0 u(t) ) \Big).
\end{align}
%

\subsubsection{Good events}
We proceed by defining the following two events in the probability space under which we can bound the terms $C_1, C_2, C_3$. We then provide a lower bound on the probability of these events. 
%
\begin{align*}
\mathcal{E}_1 = \{\Theta^0 \in \Omega\i,\, \text{for } i \ge 1\}, \quad
\mathcal{E}_2 = \{\|w(t)\| \le 2 \sqrt{p\log({ T }/{\delta})},\, \text{for } 1 \le t\le T+1\}.
\end{align*}
%
\subsubsection{Technical lemmas}
The following lemmas establish upper bounds on $C_1, C_2, C_3$.
\begin{lemma}\label{lem:C1}
Under the event $\mathcal{E}_1 \cap \mathcal{E}_2$, the following holds with probability at least $1-\delta$.
\begin{eqnarray}
C_1 \le \frac{\sqrt{128}\, C}{(1-\rho)^2} \sqrt{T}\, p \log(\frac{T}{\delta}) \sqrt{\log(\frac{1}{\delta})}\,.
\end{eqnarray}
\end{lemma}
\begin{lemma}\label{lem:C2}
Under the event $\mathcal{E}_1 \cap \mathcal{E}_2$, the following holds.
\begin{eqnarray}
C_2 \le \frac{8C}{(1-\rho)^2}  p \log (\frac{T}{\delta}) \log T\,.
\end{eqnarray}
\end{lemma}
\begin{lemma}\label{lem:C3}
Under the event $\mathcal{E}_1 \cap \mathcal{E}_2$, the following holds with probability at least $1-\delta$.
\begin{eqnarray}
|C_3| \le 800 \Big(\frac{C}{1-\rho}\Big)^{\frac{5}{2}} k \sqrt{\Big(1+\frac{k \eps^2}{(1-\rho)^2}\Big)}  
\cdot \frac{1+C}{C_{\min}}  \cdot \log(\frac{pT}{\delta})
\sqrt{\log(\frac{4kq}{\delta})} \cdot p \log T \sqrt{T}\,. 
\end{eqnarray}
\end{lemma}
\begin{lemma}\label{lem:Xbound}
The following holds true.
\begin{align}
\P(\mathcal{E}_1) \ge 1- \delta,\quad 
\P(\mathcal{E}_2 ) \ge 1- \delta.
\end{align}
Therefore, $\P(\mathcal{E}_1 \cap \mathcal{E}_2) \ge 1 - 2\delta$.
\end{lemma}

We are now in position to prove Theorem~\ref{thm:exploitation}. 
\begin{proof}[Proof (Theorem~\ref{thm:exploitation})]
Using cost decomposition (Eq.~\eqref{eqn:cost_decomposition}), under the event $\mathcal{E}_1 \cap \mathcal{E}_2$, we have
\begin{align*}
\sum_{t=0}^{T} (x(t)^{\trans} Qx(t) + u(t)^{\trans} R u(t)) &=  \sum_{t=0}^{T} J(\tTheta_t) + C_1 + C_2 +C_3\\
&\le TJ(\Theta^0) + C_1 + C_2 +C_3, 
\end{align*}
where the last inequality stems from the choice of $\tTheta_t$ by the algorithm (cf. Eq~\eqref{eqn:choiceoftheta}) and the fact that $\Theta^0 \in \Omega_t$, for all $t$ under the event $\mathcal{E}_1$. Hence, $R(T) \le C_1+C_2+C_3$\,.
Now using the bounds on $C_1,C_2,C_3$, we get the desired result. 
\end{proof}

\subsubsection*{Acknowledgments}


The authors thank the anonymous reviewers for their insightful comments.
A.J. is supported by a Caroline and Fabian Pease Stanford
Graduate Fellowship.


\newpage
\bibliographystyle{abbrv}
\bibliography{LQR_inference}

\newpage
\appendix

%

\newpage
\appendix
\section{Proof of technical lemmas}
\subsection{Proof of Lemma~\ref{th:gradie_prob_bound}}
As before let $\rho \equiv \|A^0 - B^0L\|_2 $, and $\ell = \max(1, \max_{i \in [r]}\|L_i\|_2)$.
Further, for $u \in [p]$, $j \in [q]$, define $\phi(\tau) \in \reals^{p \times p}$ to have all rows equal to zero except the $u^{th}$ row which is equal to $\tL_j{(A^0 - B^0 L)}^\tau$ . 
Define $\widetilde{\Phi}_j \in \reals^{np\times np}$ as,
\begin{equation}
\widetilde{\Phi}_j =  \left(
\begin{array}{ccccc}
0 & 0 &\ldots & 0&0\\
\phi(0)&0&\ldots&0&0\\
\phi(1)&\phi(0)&\ldots&0&0\\
\vdots&\vdots&\ddots&0&0\\
\phi(n-2)&\phi(n-3)&\ldots&\phi(0)&0\\
\end{array}
\right), 
\end{equation}
and let 
\begin{equation}
 \Phi_j = \frac{1}{2} (\widetilde{\Phi}_j + \widetilde{\Phi}_j^\trans).
  \label{eq:def_Phi}
\end{equation}
\begin{lemma}
 Let $\nu_i$ denote the $i^{th}$ eigenvalue of $\Phi_j$ and assume $\rho < 1$. 
Then,
\begin{align}
\sum^{np}_{i=1}\nu_i &= 0,\label{eq:gradient_sum_eigen_bound}\\
\max_i |\nu_i|  &\leq \frac{\ell}{1-\rho}, \label{eq:gradient_eigen_bound}\\
\sum^{np}_{i=1}\nu^2_i &\leq \frac{\ell^2 n}{2 (1-\rho)}.\label{eq:gradient_sum_eigen2_bound}
\end{align} \label{th:matrix_proper_grad}
\end{lemma}
We do not prove this lemma here and refer the reader to Lemma A.3 in \cite{bento2010learning}.

\begin{proof}[Proof (Lemma \ref{th:gradie_prob_bound})]
The proof of this lemma follows closely the proof of Proposition 4.2 in \cite{bento2010learning} which we provide here for the reader's convenience.  
Let $\mathbf{w} \in \reals^{np}$ be the vector obtained by stacking all the noise vectors up to time $n$, i.e., 
\begin{equation*}
\mathbf{w}^\trans = [w(1)^\trans, w(2)^\trans, \dotsc, w(n)^\trans].
\end{equation*}
Then we have that
\begin{align*}
  \hG_j &=  \tL_j \sum_{t=1} ^{n-1} x(t) w_u(t+1) = \sum_{t=1}^{n-1} w_u(t+1) \sum_{\tau=0}^{t-1} \tL_j (A^0 - B^0L)^\tau w(t-\tau)
   = \mathbf{w}^\trans \Phi_j \mathbf{w}.
\end{align*}
where $\Phi_j$ is defined in \eqref{eq:def_Phi}.
Since $\mathbf{w} \sim \normal(0,I_{np})$ and $\Phi_j$ is symmetric, we can write 
\begin{equation}\label{eq:hG_as_sum_zi}
\hG_j = \sum^{np}_{i=1}\nu_i z^2_i.
\end{equation}
where $z_i \sim \normal(0,1)$ are independent and $\nu_i$'s are the eigenvalues of the matrix $\Phi_j$.

Now we have for any $\beta>0$,
\begin{align*}
\prob\Big\{\sum^{np}_{i=1}\nu_i z^2_i > n \epsilon \Big\} &\leq  \exp \left(-n\beta\eps \right)\, \prod_{i=1}^{pn} \E\big\{\exp\left(\beta\nu_i z_i^2 \right) \big\}\\
& =\exp\left(-n \Big(\beta \epsilon + 
\frac{1}{2n} \sum^{np}_{i=1} \log(1-2 \nu_i \beta)  \Big)\right)\, .
\end{align*}
Let $\beta = {(1-\rho) \epsilon }/{(2 \ell^2)}$. 
Then it follows from Eq.~\eqref{eq:gradient_eigen_bound} and the assumption $\eps < \frac{\ell}{2}$  that $|2 \nu_i \beta| \leq  1/2$. 
Furthermore, for $|x|<1/2$, $\log (1-x) > -x-x^2$. 
Hence,
\begin{align*}
\prob(\sum^{np}_{i=1}\nu_i z^2_i > n \epsilon) &\leq  \exp \left(-n (\beta \epsilon - 2 \beta^2 \frac{1}{n} \sum^{np}_{i=1} \nu^2_i)\right)\\
&\leq \exp \left(- \frac{n(1-\rho)\epsilon^2 }{4 \ell^2 }\right) \, ,
\end{align*}
where the first inequality follows from the fact that $\sum^{np}_{i=1}\nu_i = 0$ (Eq.~\eqref{eq:gradient_sum_eigen_bound}) and the second inequality is obtained using the bound in Eq.~\eqref{eq:gradient_sum_eigen2_bound}.
Finally, by the union bound we obtain the desired result 
\begin{align*}
\prob\big\{\|\hG_S\|_\infty > \epsilon\big\} 
&\leq  2|S| \,\max_{j\in S} \prob\big\{z^\trans \Phi_j z > n \epsilon\big\}  \\
& \leq  2|S| \, \exp \left(- \frac{n(1-\rho) \epsilon^2 }{4 \ell^2}\right) \, .
\end{align*}
\end{proof}
%
\subsection{Proof of Lemma~\ref{en:bddQij}}
\begin{lemma}
  Let $\widetilde{R}_j \in \reals^{(n-1)p \times np}$ be obtained by removing the first $p$ rows of $\widetilde{\Phi}_j$. For $i,j \in [q]$ define $R(i,j) = 1/2 (\widetilde{R}_j^\trans \widetilde{R}_i^{} + \widetilde{R}_i^\trans \widetilde{R}_j^{}) \in \reals^{np \times np}$. 
Assume $\rho < 1$ and let $\nu_l$ denote the $l^{th}$ eigenvalue of $R(i,j)$.
Then,
\begin{align}
|\nu_l| &\leq \frac{\ell^2}{(1-\rho)^2},\\
\frac{1}{n}\sum^{np}_{l = 1} \nu^2_l &\leq \frac{2\ell^2}{(1-\rho)^3}\left(1 + \frac{3}{2n} \frac{1}{1-\rho} \right).
\end{align} \label{th:matrix_proper_hess}
\end{lemma}

\begin{proof}[Proof (Lemma~\ref{en:bddQij})]
  Our proof of \ref{th:gradie_prob_bound} here closely follows the proof of Proposition $4.2$ in \cite{bento2010learning}. 

Note that $\hH_{ij}$ can be written as,
\begin{align*}
\hH_{ij}  &=   \frac{1}{n} \sum_{t=1}^{n-1} \tL_i\ x(t) x(t)^\trans \tL_j^\trans \\
  &=   \frac{1}{n} \sum_{t=1}^{n-1} \tL_i  \Big( \sum_{\tau = 0}^{t-1} (A^0 -B^0 L)^\tau w(t-\tau)\Big)
  \Big( \sum_{\tau = 0}^{t-1} (A^0 -B^0 L)^\tau w(t-\tau)\Big)^\trans \tL_j^\trans\\
  &=   \frac{1}{n} \sum_{t=1}^{n-1}\Big( \sum_{\tau = 0}^{t-1} \tL_i  (A^0 -B^0 L)^\tau\Big)w(t-\tau)
  w(t-\tau)^\trans \Big( \sum_{\tau = 0}^{t-1} \tL_j  (A^0 -B^0 L)^\tau\Big)^\trans \\
  &=   \frac{1}{n} \sum_{t=1}^{n-1} w(t-\tau)^\trans  \Big( \sum_{\tau = 0}^{t-1} \tL_i (A^0 -B^0 L)^\tau \Big)^\trans
  \Big( \sum_{\tau = 0}^{t-1} \tL_j (A^0 -B^0 L)^\tau \Big) w(t-\tau)\\
  & =  \frac{1}{n} \mathbf{w}^\trans R(i,j) \mathbf{w}.
\end{align*}

Since $\mathbf{w} \sim \normal(0,I_{np})$ and $R(i,j)$ is symmetric, we can write 
\begin{eqnarray}
\hH_{ij}  = \frac{1}{n} \sum^{np}_{l = 1} \nu_l z_l^2\,,
\end{eqnarray}
where $z_l \sim \normal(0,1)$ are independent and $\nu_l$'s are the eigenvalues of the matrix $R(i,j)$.
Further, 
\begin{equation}
\hH_{ij} - \E (\hH_{ij}) = \frac{1}{n} \sum^{np}_{l = 1} \nu_l (z^2_l - 1), 
\end{equation}
Hence, using Chernoff bound we get
\begin{align*}
\prob ( \hH_{ij} - \E (\hH_{ij}) > \epsilon )  &= \prob( \sum^{np}_{l = 1} \nu_l (z^2_l - 1) > \epsilon n ) \\
&\leq \exp\left(-\beta \epsilon n\right) \exp\left( -\frac{1}{2} \sum^{np}_{l=1} \log (1-2 \beta \nu_l) \right).
\end{align*}
By Lemma \ref{th:matrix_proper_hess}, for $ n > \frac{3}{1-\rho}$ we have
\begin{align}
\frac{1}{n}\sum^{np}_{l = 1} \nu^2_l &\leq \frac{3\ell^2}{(1-\rho)^3},
\end{align}
Let $\beta = \frac{(1-\rho)^3 \eps}{12 \ell^2}$. 
By assumption $\eps < \frac{3}{1-\rho}$, we have $|2 \beta \nu_l| < {1}/{2}$.
Using the inequality $\log(1-x)> -x-x^2$ for $|x|<1/2$, we obtain
\begin{align*}
\prob ( \hH_{ij} - \E (\hH_{ij}) > \epsilon )  & \leq \exp\left(-\beta \epsilon n+  2 \beta^2 \sum^{np}_{l=1} \nu^2_l \right)\\
&\leq \exp\left(-\frac{n (1-\rho)^3 \eps^2}{24 \ell^2} \right),
\end{align*}
which finishes the proof.
\end{proof}

\subsection{Proof of Lemma~\ref{lem:C1}}
Before embarking on the proof, we state and prove the following claim which will be repeatedly used in the proofs of Lemmas~\ref{lem:C1},~\ref{lem:C2}, and~\ref{lem:C3}.
\begin{propo}\label{pro:x-bound}
Under the event $\mathcal{E}_1 \cap \mathcal{E}_2$, the following holds true.
\[
\|x(t)\| \le \frac{2}{1-\rho} \sqrt{p \log(\frac{T}{\delta})}\,, \text{ for } 1\le t\le T+1\,.
\]
\end{propo} 
\begin{proof}[Proof (Proposition~\ref{pro:x-bound})]{
Conditioning on the event $\mathcal{E}_1$, $\Theta^0 \in \Omega^{(i)}$ for $i\ge 1$. Furthermore, for all $i\ge 1$, $\Omega^{(i)} \subseteq \mathcal{N}_{\epsilon}(\Theta^0)$. Recall our assumption that for all $\Theta \in \mathcal{N}_{\epsilon}(\Theta^0)$, $L(\Theta)$ is identifiable with respect to $\Theta^0$. Consequently, we have $\|A^0 - B^0 L_t\|_2 \le \rho$, for all $t\ge 1$, where $L_t$ denotes the controller (used by \alg) at time $t$. Now, we write for $1\le t \le T+1$,
\begin{eqnarray}
\begin{split}
\|x(t)\| &= \|\sum_{t_1=1}^t \prod_{t_2  = t_1+1} ^t (A^0 - B^0 L_t) w(t_1)\| \le \sum_{t_1=1}^t \rho^{t-t_1} \|w(t_1)\|\\
&\le 2 \sqrt{p \log(\frac{T}{\delta})} \sum_{t_1=1}^t \rho^{t-t_1} < \frac{2}{1-\rho}\sqrt{p \log(\frac{T}{\delta})}\,,
\end{split}
\end{eqnarray}
where the second inequality holds since we are conditioning on $\mathcal{E}_2$.
}\end{proof}

Armed with this proposition, we prove Lemma~\ref{lem:C1}.

Define $z(t) = A^0 x(t) + B^0 u(t)$, and $K_t = K(\tTheta_t)$ for all $t \ge 0$. Since $x(0) = 0$, we have
\begin{align*}
C_1 &= \sum_{t=0}^{T} \bigg(x(t)^{\trans} K_t x(t) - \E\big[x(t+1)^{\trans} K_{t+1}x(t+1) \big|\F_t \big]\bigg)\\
&= - \E\big[x(T+1)^{\trans} K_{T+1}x(T+1) \big|\F_T \big] + \sum_{t=1}^{T} \bigg( x(t)^{\trans} K_t x(t) - \E \big[x(t)^{\trans}K_t x(t)|\F_{t-1}\big] \bigg) .
\end{align*}
Because $K_{T+1}$ is PSD, the first term is bounded above by zero.  To bound the second term, define
\[
\mathcal{E}_1^t = \{\Theta^0 \in \Omega_{\tau}, \text{ for } 1\le \tau \le t\}\,,
\quad 
\mathcal{E}_2^t = \{\|w(\tau)\| \le 2\sqrt{p \log(T/\delta)}, \text{ for }1\le \tau \le t  \}\,. 
\]
Note that $\mathcal{E}_1^{t+1} \subseteq \mathcal{E}_1^{t}$ and $\mathcal{E}_2^{t+1} \subseteq \mathcal{E}_2^{t}$. 
Following the approach of~\cite{abbasiregret}, it can be shown that the second term is bounded above by
\begin{align*}
\sum_{t=1}^{T} \ind_{\{\mathcal{E}_1^{t-1} \cap \mathcal{E}_2^{t-1}\}}   \bigg( x(t)^{\trans} K_t x(t) - \E \big[x(t)^{\trans}K_t x(t)|\F_{t-1}\big] \bigg)\,.
\end{align*}
Define the martingale
\[
M_{\tau} = \sum_{t=1}^{\tau} \ind_{\{\mathcal{E}_1^{t-1} \cap \mathcal{E}_2^{t-1}\}}   \bigg( x(t)^{\trans} K_t x(t) - \E \big[x(t)^{\trans}K_t x(t)|\F_{t-1}\big] \bigg), \quad M_0 = 0\,.
\]
Note that $M_{\tau}$ is a martingale since $\mathcal{E}_1^{t-1}$ and $\mathcal{E}_2^{t-1}$ are $\mathcal{F}_{t-1}$
 measurable. In addition,
 \begin{align*}
|M_{\tau} - M_{\tau-1}| &\le \ind_{\{\mathcal{E}_1^{t-1} \cap \mathcal{E}_2^{t-1}\}}\, \Big| x(t)^{\trans} K_t x(t) - \E \big[x(t)^{\trans}K_t x(t)|\F_{t-1}\big] \Big|\\
&\le 2 C\, \ind_{\{\mathcal{E}_1^{t-1} \cap \mathcal{E}_2^{t-1}\}}\, \|x(t)\|^2\\
&\le  \frac{8C}{(1-\rho)^2} {p \log(\frac{T}{\delta})}\, \ind_{\{\mathcal{E}_1^{t-1} \cap \mathcal{E}_2^{t-1}\}} \\
&\le \frac{8C}{(1-\rho)^2} {p \log(\frac{T}{\delta})}\,,
 \end{align*}
 where the penultimate inequality follows from Proposition~\ref{pro:x-bound}. Applying Azuma's inequality,
 \[
 \P(M_T - M_0 \ge \beta) \le \exp\left(- \frac{\beta^2(1-\rho)^4}{128\,T C^2 p^2 \log^2(\frac{T}{\delta})} \right)\,. 
 \]
  Hence, with probability at least $1-\delta$,
  \[
  C_1 \le \frac{\sqrt{128}\, C}{(1-\rho)^2} \sqrt{T}\, p \log(\frac{T}{\delta}) \sqrt{\log(\frac{1}{\delta})}\,.
  \]
\section{Proof of Lemma~\ref{lem:C2}}
If the confidence set is not updated at time $t+1$, i.e., $\Omega_t  = \Omega_{t+1}$, then $K(\tTheta_t) = K(\tTheta_{t+1})$ and the $t$-th term in the summation is zero. The way \alg chooses the lengths of the episodes, $\Delta \tau_i$, the number of updates (number of times \alg changes the policy) is at most $\log_4 T$ up to time $T$. Using the bound $\|K(\Theta_t)\|_2 \le C$, for $t \ge 1$, we have 
\begin{eqnarray}
\begin{split}
C_2 &= -\sum_{t=0}^T \E[x(t+1)^\trans (K(\tTheta_{t}) - K(\tTheta_{t+1})) x(t+1) | \mathcal{F}_t] \\
&\le \sum_{i: \tau_i \le T}  2C\, \|x_{\tau_i}\|^2\\
&\le \frac{8C}{(1-\rho)^2}  p \log (\frac{T}{\delta}) \log_4 T\,,
\end{split}
\end{eqnarray}
where we used Proposition~\ref{pro:x-bound} in the last step.

\section{Proof of Lemma~\ref{lem:C3}}
Let $y_t = [x_t^{\trans}, u_t^{\trans}]^{\trans} \in \reals^{q \times 1}$. We first establish the following proposition.
\begin{propo}\label{pro:C3}
Under the event $\mathcal{E}_1\cap \mathcal{E}_2$, 
The following holds true with probability at least $1 - \delta $:
\begin{eqnarray}
\sum_{t=0}^T \|(\Theta^0 - \tTheta_t) y_t\|^2 \le \frac{10}{(1-\rho)^2} p \eps^2 \log(\frac{T}{\delta}) (\log T)^2  n_1\,,
\end{eqnarray}
where $n_1$ is defined in \alg.
\end{propo}
\begin{proof}[Proof (Proposition~\ref{pro:C3})]{
Write
\begin{eqnarray*}
\begin{split}
  \sum_{t = 0}^T \|(\Theta^0 - \tTheta_t) y_t\|^2 &= \sum_{i: \tau_i \le T} \sum_{t = \tau_i}^{\tau_{i+1}-1} \|(\Theta^0 - \tTheta_t) y_t\|^2 \\
  &= \sum_{i: \tau_i \le T} \sum_{t = \tau_i}^{\tau_{i+1}-1} \|(A^0 - B^0 L^{(i)} - \tA_t + \tB_t L^{(i)}) x_t\|^2 \\
 & \le  \sum_{i: \tau_i \le T} \sum_{t = \tau_i}^{\tau_{i+1}-1} 2\|(A^0-B^0 L\i - \tA_t+ \tB_t L\i) (A^0-B^0 L\i)^{t-\tau_i+1} x(\tau_i-1)\|^2 \\
 & \qquad + \sum_{i: \tau_i \le T} \sum_{t = \tau_i}^{\tau_{i+1}-1} 2 \|(A^0-B^0 L\i - \tA_t+ \tB_t L\i) \sum_{t_1 = \tau_i}^t (A^0-B^0 L^{(i)})^{t-t_1} w(t_1) \|^2\,,
\end{split}
\label{eq:predict_err_decomp}
\end{eqnarray*}
where we used 
\[
x(t) = (A^0 - B^0L^{(i)})^{t - \tau_i+1} x(\tau_i - 1)  + \sum_{t_1 = \tau_i}^t (A^0 -B^0 L^{(i)})^{t-t_1} w(t_1).
\]
We proceed by bounding the first term as follows:
\begin{eqnarray}\label{eq:bound-term-1}
\begin{split}
&\sum_{i: \tau_i \le T} \sum_{t = \tau_i}^{\tau_{i+1}-1} 2\|(A^0-B^0 L\i - \tA_t + \tB_t L\i) (A^0-B^0 L\i)^{t-\tau_i+1} x(\tau_i-1)\|^2\\
&\le \sum_{i: \tau_i \le T} \sum_{t = \tau_i}^{\tau_{i+1}-1} 2 \dist(\Theta^0, \tTheta_t)^2 \rho^{2(t - \tau_i +1)} \|x(\tau_i -1)\|^2\\
&\le \frac{8}{(1-\rho)^2} p \log(\frac{T}{\delta}) \sum_{t=0}^T  \dist(\Theta^0, \tTheta_t)^2\,.
\end{split}
\end{eqnarray}

To bound the second term define the matrix
\begin{equation}
 D_t = (A^0-B^0 L\i - \tA_t + \tB_t L\i) [I, (A^0-B^0 L)^{ 1 }, (A^0-B^0 L)^{2}, \dots, (A^0-B^0 L)^{t-\tau_i}]. 
\end{equation}
The second term can be written as $\sum_{t=0}^T 2\|D_t \mathbf{w}_i\|^2$ where $\mathbf{w}_i$ is the vector obtained by stacking all the noise vectors in episode $i$, i.e., 
\begin{equation*}
\mathbf{w}_i^\trans = [w(t)^\trans, w(t -1 )^\trans, \dotsc, w(\tau_{i})^\trans]^\trans\,.
\end{equation*}
Hence, the $r^{th}$ entry in the vector $D_t \mathbf{w}_i$ is a normal random variable with variance at most $\|D_{tr}\|^2$ where $D_{tr}$ is the $r^{th}$ row of matrix $D_t$ and
\begin{equation}
  \|D_{tr}\|^2 \le \frac{\d(\Theta^0, \tTheta_t)^2}{(1-\rho^2)} . 
\end{equation}
Using standard normal tail bound  we get
\begin{equation}
  \prob((D_{tr}\mathbf{w}_i)^2 \ge g_t ) \le \exp\Big(-\frac{(1-\rho^2)}{2 \d(\Theta^0, \tTheta_t)^2} g_t\Big). 
\end{equation}
Taking 
\begin{equation}
  g_t =  \frac{2\d(\Theta^0, \tTheta_t)^2}{(1-\rho^2)} \log(\frac{ p T}{\delta})\,,
\end{equation}
and applying union bound for $r \in [p]$, and $1\le t \le T$, we obtain 
\begin{equation}
  \prob\left((D_{tr}\mathbf{w}_i)^2 \le  g_t , \;\text{for}\; 1\le t \le T, r\in[p] \right) \ge 1 - \delta\,.
\end{equation}
Consequently, with probability at least $1-\delta$, the second term is bounded by
\begin{eqnarray}\label{eq:bound-term-2}
2\sum_{t=0}^T \|D_t \mathbf{w}_i\|^2 \le 2\sum_{t=0}^T p g_t \le \frac{4}{(1-\rho^2)}p \log(\frac{ p T}{\delta}) \sum_{t=0}^T \dist(\Theta^0, \tTheta_t)^2.
\end{eqnarray}
}
Finally, using Theorem~\ref{thm:learning} and the choice of episodes, we have
\[
\sum_{t=0}^T \dist(\Theta^0, \tTheta_t)^2  = \sum_{i:\tau_i \le T} (2^{-i}{\eps})^2 \Delta \tau_i
= \sum_{i:\tau_i \le T} (2^{-i}{\eps})^2 4^i \Big(1+ \frac{i}{\log(\frac{q}{\delta})}\Big) n_1
\le \eps^2 \Big(\log T + \frac{\log^2 T}{2 \log(\frac{q}{\delta})}\Big) n_1\,, 
\]
where we have used the fact that the number of episodes up to time $T$ is at most $\log T$.
Combining Eqs.~\eqref{eq:bound-term-1} and~\eqref{eq:bound-term-2}, we have
\begin{eqnarray}
\begin{split}
 \sum_{t = 0}^T \|(\Theta^0 - \tTheta_t) y_t\|^2
 &\le \Big\{\frac{8}{(1-\rho)^2} p \log(\frac{T}{\delta}) + \frac{4}{(1-\rho^2)}p \log(\frac{ p T}{\delta}) \Big\} 
 \sum_{t=0}^T  \dist(\Theta^0, \tTheta_t)^2\\
 &\le \frac{12}{(1-\rho)^2} p \log(\frac{pT}{\delta}) \eps^2 \Big(\log T + \frac{\log^2 T}{2 \log(\frac{q}{\delta})}\Big) n_1\\
 &\le \frac{20}{(1-\rho)^2} p \eps^2 \log(\frac{pT}{\delta}) (\log T)^2  n_1\,.
\end{split}
\end{eqnarray}
\end{proof}
\begin{coro}\label{coro:C3}
Using the value of $n_1$, defined in \alg, we have with probability at least $1-\delta$,
\begin{eqnarray}
\sum_{t= 0}^T \|(\Theta^0 - \tTheta_t) y_t\|^2 
\le \frac{8 \cdot 10^4 C^2}{(1-\rho)^3 C_{\min}^2} p k^2 \Big(1+\frac{k \eps^2}{(1-\rho)^2} \Big) \log(\frac{4kq}{\delta}) \log(\frac{pT}{\delta}) \log^2 T\,.
\end{eqnarray}
\end{coro}
Now, we are ready to bound $C_3$.
\begin{eqnarray}\label{eqn:C3_bound}
\begin{split}
|C_3| &\le \sum_{t=0}^{T}  \bigg|\|K(\tTheta)^{1/2} \tTheta_t ^{\trans} y_t\|^2 -  \|K(\tTheta)^{1/2} \Theta^{0\trans} y_t\|^2\bigg|\\
&\le \bigg(\sum_{t=0}^T  \bigg\{\|K(\tTheta_t)^{1/2} \tTheta_t y_t\| - \|K(\tTheta_t)^{1/2} {\Theta}^0 y_t\|\bigg\}^2\bigg)^{1/2} \times\\&\quad \quad
\bigg(\sum_{t= 0}^T \bigg\{\|K(\tTheta_t)^{1/2} \tTheta_t y_t\| + \|K(\tTheta_t)^{1/2} {\Theta}^0 y_t\|\bigg\}^2\bigg)^{1/2} \\
&\le \bigg(\sum_{t= 0}^T  \|K(\tTheta_t)^{1/2} (\tTheta_t - {\Theta}^0) y_t\|^2\bigg)^{1/2} \times\\
&\quad \quad
\bigg(\sum_{t=0}^T  \bigg\{\|K(\tTheta_t)^{1/2} \tTheta_t y_t\| + \|K(\tTheta_t)^{1/2} {\Theta}^0 y_t\|\bigg\}^2\bigg)^{1/2} \\
&\le C^{1/2}\bigg(\sum_{t= 0}^T \|(\tTheta_t - {\Theta}^0) y_t\|^2\bigg)^{1/2} \times C  \bigg(\sum_{t=0}^T \|y_t\|^2 \bigg)^{1/2}.
\end{split}
\end{eqnarray}
Corollary~\ref{coro:C3} provides an upper bound for the first term on the right hand side. In addition,
\begin{eqnarray}
\begin{split}
\sum_{t= 0}^T  \|y_t\|^2 &\le \sum_{t= 0}^T (1+\sigma(L_t)^2) \|x_t\|^2\\
 &\le (1+\sigma(L_t)^2) \frac{4}{(1-\rho)^2}p \log(\frac{T}{\delta}) T \\
 &\le \frac{4(1+C^2)}{(1-\rho)^2} p \log(\frac{T}{\delta}) T\,.
\end{split}
\end{eqnarray} 
Here, the first inequality follows from Proposition~\ref{pro:x-bound}.
Combining the bounds for the terms on the right hand side of Eq.~\eqref{eqn:C3_bound}, we obtain
\begin{eqnarray}
|C_3| \le 800 \Big(\frac{C}{1-\rho}\Big)^{\frac{5}{2}} k \sqrt{\Big(1+\frac{k \eps^2}{(1-\rho)^2}\Big)}  
\cdot \frac{1+C}{C_{\min}}  \cdot \log(\frac{pT}{\delta})
\sqrt{\log(\frac{4kq}{\delta})} \cdot p \log T \sqrt{T}\,. 
\end{eqnarray}
%

\section{Proof of Lemma~\ref{lem:Xbound}}
We first show that $\P(\mathcal{E}_1) \ge 1 -\delta$. According to Theorem~\ref{thm:learning}, the sample complexity scales with $(1/\eps^2) \log(q/\delta)$. Due to the choice of episode lengths in the algorithm, namely $\Delta \tau_i =  4^i (1+i/\log(q/\delta)) n_1$, with probability at least $1-\delta/2^i$, we have $\dist(\Theta_0,\hTh^{\i}) \le 2^{-i}\eps$ and thus $\Theta_0 \in \Omega^{\i}$.

Now by applying union bound for $i \ge 1$,   
\begin{eqnarray}
\P(\mathcal{E}_1) \ge 1-\sum_{i=1}^\infty \frac{\delta}{2^i} = 1 - {\delta}.
\end{eqnarray}

Next we prove the lower bound for the probability of event $\mathcal{E}_2$.
%
Let $w(t) \in \reals^p$ be the noise vector at time $t$ with i.i.d standard normal entries. 
For any $t \ge 1$ and any $\lambda$, we have
\begin{eqnarray}
\begin{split}
\P\{\|w(t)\|^2 \ge \lambda p \} &= \P\{e^{\theta \sum_{i=1}^p w_i(t)^2} \ge e^{\theta \lambda p}\}\\
&\le e^{-\theta \lambda p} \prod_{i=1}^p \E\{e^{\theta w_i^2(t)}\} \\
&=\exp(-p\{\lambda \theta + \frac{1}{2} \log(1-2\theta)\})\\
&\le \exp(-p\{\lambda \theta -\theta -2\theta^2\}),\quad\quad \text{for } 0< \theta <1
\end{split}
\end{eqnarray}
where we used the fact that if $|x| <1$, then $\log(1-x) > -x -x^2$. Choosing $\theta = 1/2$, and $\lambda = 4 \log(T/\delta)$, we obtain
\begin{eqnarray*}
\P\{\|w(t)\|^2 \ge  4 p \log({T}/{\delta}) \} \le \exp(-p \log({T}/{\delta})) = (\frac{T}{\delta})^{-p}.
\end{eqnarray*}
Finally, by applying union bound for $1\le t \le T+1$,
\begin{eqnarray}
\begin{split}
\P(\mathcal{E}_2) &= \P\{\|w(t)\| \le 2\sqrt{{p} \log(T/{\delta})},\, \text{ for } 1\le t\le T+1 \} \\
&\ge 1 - (T+1) (\frac{T}{\delta})^{-p} > 1- \delta\,.
\end{split}
\end{eqnarray}

%
%
%


\end{document}